\documentclass{article}

   \usepackage[preprint]{neurips_2019}

\usepackage[utf8]{inputenc} 
\usepackage[T1]{fontenc}    
\usepackage{url}            
\usepackage{booktabs}       
\usepackage{graphicx}
\usepackage{amsfonts}       
\usepackage{nicefrac}       
\usepackage{microtype}      
\usepackage{amsmath}
\usepackage{amssymb}
\usepackage{amsthm}
\usepackage{bm}
\usepackage{wrapfig}
\usepackage{xcolor}
\usepackage{hyperref}  
\definecolor{mydarkblue}{rgb}{0,0.08,0.45}
\definecolor{myfavblue}{rgb}{0.1176, 0.392, 1.0}
\hypersetup{
    colorlinks=true,
    linkcolor=mydarkblue,
    citecolor=mydarkblue,
    filecolor=mydarkblue,
    urlcolor=mydarkblue}
\usepackage[capitalise]{cleveref}

\newtheorem{lemma}{Lemma}[section]
\newtheorem{theorem}{Theorem}[section]
\newtheorem{definition}{Definition}[section]
\newtheorem{proposition}{Proposition}[section]

\newtheorem{corollary}{Corollary}[section]

\setcitestyle{authoryear,open={(},close={)}}

\title{Universality Theorems for Generative Models}

\author{
  Valentin Khrulkov \\
  Skolkovo Institute of \\
  Science and Technology\\
  Moscow, Russia\\
  \texttt{valentin.khrulkov@skoltech.ru} \\
  \And
  Ivan Oseledets  \\
  Skolkovo Institute \\
  of Science and Technology\\
  Moscow, Russia\\
  Institute of Numerical Mathematics,\\
  Russian Academy of Sciences \\
  Moscow, Russia \\
  \texttt{i.oseledets@skoltech.ru} 
}

\begin{document}

\maketitle

\begin{abstract}
Despite the fact that generative models are extremely successful in practice, the theory underlying this phenomenon is only starting to catch up with practice. In this work we address the question of the \emph{universality} of generative models: is it true that neural networks can approximate any data manifold arbitrarily well? We provide a positive answer to this question and show that under mild assumptions on the activation function one can always find a feedforward neural network which maps the latent space onto a set located within the specified Hausdorff distance from the desired data manifold. We also prove similar theorems for the case of multiclass generative models and cycle generative models, trained to map samples from one manifold to another and vice versa.
\end{abstract}
\section{Introduction}
Generative models such as Generative Adversarial Networks (GANs) are widely used for tasks such as image synthesis, semi-supervised learning, and domain adaptation \citep{brock2018large,radford2015unsupervised,zhang2017stackgan,isola2017image}. Such generative models are trained to perform a mapping from a latent space of a small dimension to some specified data manifold, typically represented by a dataset of natural images. Despite their success and excellent performance, the theory behind such models is not yet well understood. A recent survey of open questions about generative models \citep{odena2019open} among others presents the following question: what sorts of distributions can GANs model? In particular, what does it even mean for a GAN to \emph{model a distribution}? 

To answer these questions we adopt the following geometric approach, very amenable to precise mathematical analysis.  Under the assumption of the \emph{Manifold Hypothesis} \citep{goodfellow2016deep}, data comes from a certain data manifold. Then the goal of a generator network is to reproduce this data manifold as closely as possible by mapping the latent space into the ambient space of the data manifold. This intuitive understanding can be written in a more concrete manner as follows.
Suppose that we are given the latent space $\mathcal{M}_z$, feedforward neural network $f_{\theta}$ as a generator, and some target data manifold $\mathcal{M}$. In order for the manifold $\mathcal{M}$ to be generated by $f_{\theta}$ we require that the image of $\mathcal{M}_z$ under $f_{\theta}$ is sufficiently close to $\mathcal{M}$, more specifically that the Hausdorff distance between $f_{\theta}(\mathcal{M}_z)$ and $\mathcal{M}$ is less than the given parameter $\varepsilon$. Hausdorff distance is a well-defined metric on the space of all compact subsets of Euclidean space and hence is equal to zero if and only if $f_{\theta}(\mathcal{M}_z) = \mathcal{M}$ --- the case of precise replication of the data manifold. Thus, the question at hand can be formulated as follows: is it possible to approximate in the sense of the Hausdorff distance an arbitrary compact (connected) manifold using standard feedforward neural networks? By combining techniques from Riemannian geometry with well--known properties of neural networks we provide a positive answer to this question. We also show that the condition of being \emph{smooth} is not necessary and the results are also valid for just topological manifolds. 

We further extend the discussed geometric approach for the theoretical analysis of many practical situations, for instance, to the case of data manifolds, which consist of multiple disjoint manifolds and correspond to multiclass datasets, and cycle generative models \citep{zhu2017unpaired,isola2017image}, which for two manifolds learn an approximately invertible mapping from one manidold to another. For the latter case we prove a somewhat surprising result that for \emph{any} given pair of data manifolds of the same dimension, one can always train a pair of neural networks which are approximately inverses of one another, and map the first manifold \emph{almost} onto the second one, and vice versa. In this work, we ignore specifics of the training algorithm (for instance, what loss function is used) and merely focus on understanding the generative capabilities of neural networks.
\section{Related work}
A large body of papers is devoted to analyzing the universality of neural networks. Classical works on universality \citep{cybenko1989approximation,hornik1991approximation,haykin1994neural,hassoun1995fundamentals} prove that neural networks with one hidden layer are \emph{universal approximators} and can approximate arbitrary continuous functions on compact sets. Similar results also stand for deep wide networks with ReLU nonlinearities \citep{lu2017expressive}, convolutional neural networks \citep{cohen2016convolutional} and recurrent neural networks \citep{khrulkov2019generalized}.

GANs were mostly studied from point of view of convergence properties \citep{feizi2017understanding,balduzzi2018mechanics,lucic2018gans}. Several works focus on the relation between geometric properties of datasets and behavior of GANs. In order to analyze what characteristics of datasets lead to better convergence, synthetic datasets were studied in \citep{lucic2018gans}. A case of disconnected data manifold (similar in spirit to our analysis in \cref{sec:main}) was analyzed in \citep{khayatkhoei2018disconnected}. A metric for analyzing the quality of GANs based on comparing geometric properties of the original and generated datasets was proposed in \citep{khrulkov2018geometry}.

\section{Notation and assumptions}
We will denote the $d$-cube $[-1, 1]^d$ by $I_d$. We will often use an approximation of a continuous function by a neural network, in that case, the ``network version'' of the function will be indicated by a subscript $\theta$ or $\phi$ indicating a collection of trainable parameters, e.g., $f_{\theta}$ or $g_{\phi}$. 

In this work, we deal with data manifolds. We assume that all these manifolds are smooth, orientable, compact and connected unless stated explicitly. We also assume that all the manifolds are embedded into a Euclidean space $\mathbb{R}^n$, and inherit the Riemannian metric tensor $g$. By \emph{smooth} we will mean infinitely differentiable manifolds (functions), i.e, of class $C^{\infty}$; all the results, however, will stay true if we consider class $C^{r}$ for some finite $r$.  As a norm of a function $f$ defined on some compact set $D$ we will use the $C$-norm: $\|f\|_D = \max_{x \in D} |f(x)|$, and for vectors we use the $2$-norm. 

We will often make use of a natural geometric measure $\mu$ on a manifold, which can be constructed by integrating the \emph{volume form} associated with the Riemannian metric tensor over the corresponding set.
\section{Background}\label{sec:background}
Let us first present some background material necessary for understanding the proofs. We will freely use the term \emph{manifold} in the precise mathematical sense. Due to limited space, we do not provide the definition and refer the reader to thorough introductions such as \citep{lee2013smooth,sakai1996riemannian}. 

First important construction in the proof is the \emph{exponential map}.
\subsection{Exponential map}
Let $\mathcal{M}$ be a Riemannian manifold endowed with a metric tensor $g$. Recall that \emph{geodesics} are locally length minimizing curves, defined as a solution of a certain second-order differential equation. An important property of geodesics is that the length of the velocity vector is preserved along the curve, i.e., for a geodesic $\gamma(t)$ we have 
\begin{equation}\label{eq:const_vel}
    \frac{d}{dt} \|\dot{\gamma}(t)\| = 0.
\end{equation}.

The \emph{exponential map} is defined in the following manner. Let $q \in \mathcal{M}$ and $v \in T_p \mathcal{M}$, and suppose that there exists a geodesic $\gamma: [0,1] \to \mathcal{M}$ with $$\gamma(0) = q, \quad \dot{\gamma}(0) = v .$$
Then the point $\gamma(1) \in \mathcal{M}$ is denoted by $\exp_q(v)$ and called the exponential of the tangent vector $v$. The geodesic $\gamma$ can then be written as $\gamma(t) = \exp_q{(vt)}$.
While apriori the exponential map is defined only if $\|v\|$ is small enough, for certain class of manifolds it is globally defined. Namely, if a manifold is \emph{geodesically complete}, then $\exp_q(v)$ is defined for all $q$ and $v \in T_q \mathcal{M}$. Our proof is based on the following classical result.
\begin{theorem}[Hopf-Rinow]\label{thm:hopf}
Let $(\mathcal{M}, g)$ be a connected Riemannian manifold. Then the following statements are equivalent. 
\begin{itemize}
    \item The closed and bounded subsets of $\mathcal{M}$ are compact;
    \item $\mathcal{M}$ is a complete metric space;
    \item $\mathcal{M}$ is geodesically complete.
\end{itemize}
Furthermore, any of the above implies that any points $p$ and $q$ in $\mathcal{M}$ can be connected by a minimal (length--minimizing) geodesic.
\end{theorem}
In particular, this implies that any compact connected manifold $\mathcal{M}$ is geodesically complete.
\subsection{Hausdorff distance}
The \emph{Hausdorff distance} between two sets $X, Y \subset \mathbb{R}^n$ is defined as follows.
\begin{equation}\label{eq:hausdorff}
d_{H}(X, Y)=\inf \left\{\varepsilon \geq 0 ; X \subseteq [Y]_{\varepsilon} \text { and } Y \subseteq [X]_{\varepsilon}\right\},
\end{equation}
where
\begin{equation}\label{eq:nbrhd}
 [X]_{\varepsilon} :=\bigcup_{x \in X}\{z \in \mathbb{R}^n ; d(z, x) \leq \varepsilon\}.   
\end{equation}
It is well--known that the set of all compact subsets of $\mathbb{R}^n$ endowed with the Hausdorff distance becomes a complete metric space \citep{henrikson1999completeness}.
\subsection{Universal Approximation Property of Neural Networks}
In this paper we heavily rely on the following classical results on neural networks \citep{cybenko1989approximation, hornik1991approximation}.
\begin{theorem}[Universal Approximation Theorem]\label{thm:uat}
Let $\phi:\mathbb{R}\to\mathbb{R}$ be a nonconstant, bounded and continuous function. Then for any continuous function $f:I_n \to \mathbb{R}$ and $\varepsilon>0$ there exists a fully connected neural network $f_{\theta}$ with the activation function $\phi$ and one hidden layer, such that
$$
\max_{x \in I_n}|f(x) - f_{\theta}(x)| < \varepsilon.
$$
\end{theorem}
In our analysis we restrict ourselves to the case of neural networks of the form considered in \cref{thm:uat}. However, all the results stand for any other learnable parametric maps with the property of being dense in the space of continuous functions.
\section{Geometric Universality Theorem}\label{sec:main}
In this section we prove that for an arbitrary manifold it is possible to construct a neural network, mapping the cube $I_d$ approximately onto this manifold. 
Our analysis is based on the following lemma. In fact, this is a particular case of a much stronger theorem valid even for \emph{topological} manifolds (without smooth structure), for which we provide a discussion and reference further in the text. We, however, believe that this particular case is instructive and provides an intuition on how the generative mappings may look like.
\begin{lemma}\label{lemma:map}
Let $\mathcal{M} \subset \mathbb{R}^n$ be a compact connected $d$-dimensional manifold. Then there exists a smooth map $$f:I_d \to \mathbb{R}^n,$$ such that $f(I_d) = \mathcal{M}$.
\end{lemma}
\begin{proof}
We will construct this map explicitly. Choose an arbitrary point $q \in \mathcal{M}$, and consider
$$\exp_q:T_q \mathcal{M} \to \mathcal{M}.$$
Since $\mathcal{M}$ is compact and connected, it is geodesically complete and the Hopf-Rinow theorem applies. Thus, this map is defined on $T_q \mathcal{M} \cong \mathbb{R}^d$ and surjective.

We now need to show that we can choose a compact subset of $T_q \mathcal{M}$ such that the restriction of $\exp_q$ to this subset is also surjective. 
To do this observe that since $\mathcal{M}$ is compact it has finite \emph{diameter}, namely $\forall p, q: d(p, q) \leq R_0$ for some finite constant $R_0$. Here $d$ is the Riemannian distance, defined as the arc length of a minimizing geodesic. From \cref{eq:const_vel} it instantly follows that for the (Euclidean) ball $B_{R_0} = \lbrace v \in T_q \mathcal{M} : \|v\| \leq R_0 \rbrace$ we have $\exp_q(B_{R_0}) = \mathcal{M}$. Indeed, since any point on $\mathcal{M}$ is within distance $R_0$ from $q$, there exists a minimal geodesic connecting these points with length bounded by $R_0$. But for any vector $v \in T_q \mathcal{M}$ from \cref{eq:const_vel} we obtain that the length of the corresponding geodesic connecting $q$ and $\exp_q(v)$ is exactly $\|v\|$, which proves the claim.
Statement of the lemma then follows after selecting an arbitrary cube containing $B_{R_0}$ and appropriate rescaling.
\end{proof}
\begin{wrapfigure}{r}{0.5\textwidth}%
	\centering
	\includegraphics[width=0.49\textwidth]{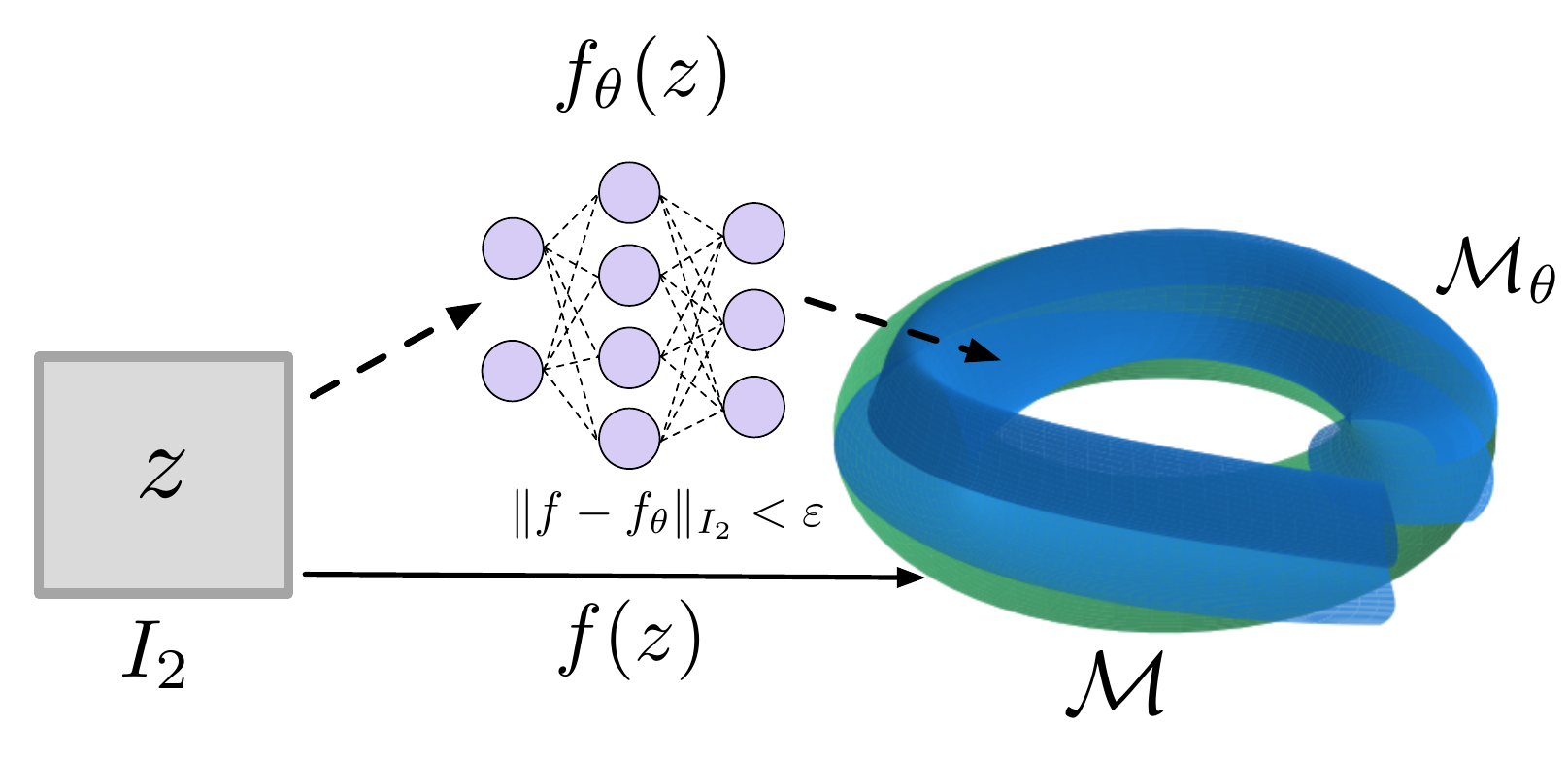}
	\caption{Visualization of the construction in the proof of \cref{thm:universality1}. The latent space $I_2$ is mapped onto the manifold $\mathcal{M}$ via the function $f$. This mapping is then approximated via neural network $f_{\theta}$, which in turn maps $I_2$ onto the compact set $\mathcal{M}_{\theta}$. If $f_{\theta}$ is sufficiently close to $f$ then so are $\mathcal{M}$ and $\mathcal{M}_{\theta}$.}
	\vspace{-10pt}
	\label{fig:manifolds}
\end{wrapfigure}%
\begin{theorem}[Geometric Universality of Generative Models]\label{thm:universality1}
Let $\mathcal{M}$ be a compact connected $d$-dimensional manifold. For every nonconstant, bounded, continuous activation function $\varphi: \mathbb{R} \to \mathbb{R}$ and $\varepsilon > 0$ there exists a fully connected neural network $f_{\theta}(z): I_d \to \mathbb{R}^n$ with the activation function $\varphi$, such that $d_H(\mathcal{M}, \mathcal{M}_{\theta}) < \varepsilon $. Here $\mathcal{M}_{\theta} =  f_{\theta}(I_d)$.
\end{theorem}
\begin{proof}
Choose an arbitrary $f$ as in \cref{lemma:map}. By the standard universal approximation theorem for neural networks we can find such a neural network that $\|f - f_{\theta}\|_{I_d} < \varepsilon$. Statement of the theorem then follows from the definition of the Hausdorff distance. Indeed, by surjectivity of $f$ we find that every point $x_0 = f(z_0) \in \mathcal{M}$ is within distance $\varepsilon$ from the point $f_{\theta}(z_0) \in \mathcal{M}_{\theta}$, and thus $\mathcal{M} \subset [\mathcal{M}_{\theta}]_{\varepsilon}$ as in \cref{eq:nbrhd}, and conversely $\mathcal{M_{\theta}} \subset [\mathcal{M}]_{\varepsilon}$. See \cref{fig:manifolds} for illustration of the proof.
\end{proof}
Previously we have noted that our \cref{lemma:map} is a particular case of a much stronger result \citep{brown1962mapping}. Namely, it can be stated as follows.
\begin{lemma}[Brown's mapping theorem]\label{lemma:berlanga}
	Let $\mathcal{M}$ be a compact connected $d$-dimensional \emph{topological} manifold. Then there exists a continuous map
	$$f:I_d \to \mathbb{R}^n,$$ such that
	$f(I_d) = \mathcal{M}$.
\end{lemma}
Based on this lemma \cref{thm:universality1} can be generalized to include the more general case of topological data manifolds. 
\begin{corollary}[Geometric Universality for Topological Manifolds]
	\Cref{thm:universality1} holds true for $\mathcal{M}$ being an arbitrary compact connected topological manifold.
\end{corollary}

\paragraph{Multiclass case}
The previous theorem considers only the case of a single data manifold. However, commonly in practice, single datasets contain samples from multiple data manifolds (e.g, MNIST digits, ImageNet classes). Since we can assume that these manifolds do not intersect, it is impossible to map a connected latent space surjectively onto this disconnected joint data manifold. To counteract this effect we can allow small pieces of latent space to map into thin ``tunnels'' connecting those manifolds. This can be made precise by the following statement.
\begin{theorem}[Geometric Universality for Multiclass Manifolds]
	Let $\mathcal{M} = \sqcup_{i=1}^c \mathcal{M}_i$ be a ``multiclass'' data manifold, with each $\mathcal{M}_i$ being a compact connected $d$-dimensional topological manifold. Then for every $\delta$ and every nonconstant, bounded, continuous activation function $\varphi:\mathbb{R} \to \mathbb{R}$ there exists a fully connected neural network $f_{\theta}(z):I_d\to \mathbb{R}^n$ with the activation function $\varphi$ such that the following properties hold.
	\begin{itemize}
		\item There exists a collection $\lbrace D_i \rbrace_{i=1}^c$ of disjoint compact subsets of $I_d$ such that 
		\begin{equation}
			\forall i \ d_H(f_{\theta}(D_i), \mathcal{M}_i) < \varepsilon.
		\end{equation}
		
		\item $\mu(I_d \setminus \sqcup_{i=1}^c D_i) \leq \delta.$
	\end{itemize} 
	
\end{theorem}
\begin{proof}
Similar to the proof of \cref{thm:universality1} we will apply the universal approximation theorem to a certain function constructed with the help of \cref{lemma:berlanga}. To construct such function let us select sets $D_i$ in the following way. We divide the interval $[-1, 1]$ uniformly into $c$ intervals, namely $[x_0, x_1], [x_1, x_2], \hdots ,[x_{c-1}, x_c]$ with length of each interval being $\frac{1}{c}$ and $x_0=-1,x_c=1$. We propose to use the following $D_i$, satisfying conditions of the corollary. Denote $h = \frac{\delta}{2(c-1)}$,
\begin{equation}
D_i = \begin{cases}
\left[x_i, x_{i+1} - h\right] \times [-1, 1]^{d-1},\ i=0, \\ 
\left[x_i + h, x_{i+1} - h\right] \times [-1, 1]^{d-1}, 0<i<c-1,\\
  \left[x_i + h, x_{i+1}\right] \times [-1, 1]^{d-1}, i=c-1.
\end{cases}
\end{equation}
Intuition is very simple: we chop down the cube $D$ on the first axis into smaller boxes, and remove some space between them. On each of the chunks $D_i$ we can now apply \cref{lemma:berlanga} for the corresponding manifold $\mathcal{M}_i$, obtaining a collection of maps $\lbrace f_i \rbrace_{i=1}^{c}$. To construct a global continuous map $f$ we can now simply linearly interpolate each of the maps $f_i$ from the right boundary \mbox{$[x_{i+1} - h]\times [-1, 1]^{d-1}$} of one box to the left boundary $[x_{i+1} + h]\times [-1, 1]^{d-1}$ of the neighboring one. By applying the universal approximation theorem to this function $f$, we finalize the proof.
\end{proof}

\section{Invariance property of deep expanding networks}\label{thm:universality}
Our previous results state that it is possible to approximate any given manifold $\mathcal{M}$ up to some accuracy. However, neural networks used in the proof are shallow (they have one hidden layer) and are not practical. In this section, we study how the set $\mathcal{M}_{\theta}$ looks like for more practical networks consisting of a series of fully connected and convolutional layers. We will show a somewhat surprising result that under certain mild conditions such networks cannot significantly transform the latent space, more precisely the generated set $\mathcal{M}_{\theta}$ will be diffeomorphic to the open unit cube $(-1, 1)^d$. In fact, our results will be more general and will demonstrate that this property holds for arbitrary latent spaces, that is if $z$ is sampled from some manifold $\mathcal{M}_z$, then $\mathcal{M}_{\theta}$ will be diffeomorphic to $\mathcal{M}_z$.

\subsection{Reminder on embeddings}
Recall the following definition. 
\begin{definition}[Smooth embedding]
Let $\mathcal{M}$ and $\mathcal{N}$ be smooth manifolds and $f:\mathcal{M}\to\mathcal{N}$ be a smooth map. Then $f$ is called an embedding is the following conditions hold.
\begin{itemize}
	\item Derivative of $f$ is everywhere injective;
	\item $f$ is an injective, continuous and open map (i.e, maps opens sets to open sets).
\end{itemize}
\end{definition}
The main property of a smooth embedding is the following \citep{lee2013smooth}.
\begin{proposition}\label{prop:embedding}
The domain of an embedding is diffeomorphic to its image.
\end{proposition}
We will show that certain neural networks commonly used for generative models are in fact smooth embeddings, and thus their image is diffeomorphic to the domain (latent space). We analyze two most commonly used layers in such models: fully connected and convolutional layers (both standard and transposed). For the sake of simplicity we assume that convolutions are \emph{circularly} padded, i.e., the input presents a two-dimensional torus; in this case, when the offset calls for a pixel that is off the left end of the image, the layer ``wraps around'' to take it from the opposite end. We consider arbitrary stride, in order to allow for a layer to increase the spatial size of a feature tensor, as commonly done. 

Let us fix the nonlinearity $\sigma(z)$ to be an arbitrary smooth monotonous function without saddle points ($\sigma'(z) \neq 0$). Then the following two lemmas hold. Let us first assume that the latent space is the Euclidean space $\mathbb{R}^d$ (or equivalently, an open unit cube $(-1, 1)^d$). 
\begin{lemma}\label{lemma:linear}
Let $f(z) = \sigma(Az + b)$ with $A \in \mathbb{R}^{n \times m}$ be a fully connected layer. If $n \geq m$ then $f(z)$ is a smooth embedding for all $A$ except for a set of measure zero. We will call such a layer an \textbf{expanding fully connected layer}.
\end{lemma}
\begin{proof}
Indeed, such a map is injective. It is open as a composition of a linear map (which is trivially open), and of $\sigma(z)$ which is open since it is a continuous monotonous function. Then for all matrices $A$ of full rank (which form a set of full measure in the space of matrices of size $n \times m$) the derivative is injective by a simple application of the chain rule and the fact that $\sigma'(z) \neq 0$.
\end{proof}
Let us now deal with the convolutional layers. 
\begin{lemma}\label{lemma:conv}
Let $z$ be a $3$rd--order tensor tensor representing a feature tensor of size $m \times m$ with $k$ channels. Suppose that $f(z) = \sigma(\mathrm{Conv}(z) + b)$ is a standard convolutional or transposed convolutional layer with an arbitrary stride. Suppose that $\mathrm{Conv}$ is parameterized via a kernel parameter $C \in \mathbb{R}^{l \times k \times s \times s}$, such that $f(z)$ is a feature tensor of size $n \times n$ with $l$ channels. If $n^2l \geq m^2k$ then $f(z)$ is a smooth embedding for all $C$ except for a set of measure zero. We will call such a layer an \textbf{expanding convolutional layer}. 
\end{lemma}
\begin{proof}
The only non-trivial part of the proof is showing injectivity of this layer for all $C$ but measure zero. Note that if $n^2l \geq m^2k$ then the matrix representing the linear map performing the $\mathrm{Conv}$ operation is vertical, hence it is sufficient to show that generically it is of full rank. In the case of the transposed convolution, we can transpose this matrix and analyze the corresponding convolutional layer. 
\paragraph{Stride one}
Let us start with the most important case of stride being one, in which case $m=n$. Denote the matrix of the linear map underlying $\mathrm{Conv}$ by $\widehat{C} \in \mathbb{R}^{n^2l \times n^2k}$, that is $\mathrm{vec}(\mathrm{Conv}(x)) = \widehat{C} \mathrm{vec}(x)$, where $\mathrm{vec}$ denotes the \emph{vectorization} operator. We need to show that for all $C$ but measure zero this matrix is of full rank.

To prove the lemma we use the following simple argument coming from algebraic geometry. The condition of matrix $\widehat{C}$ \emph{not} being a full rank is \emph{algebraic} (i.e., is given by polynomial equations) in the space of parameters $C$. Indeed, the operation of constructing $\widehat{C}$ based on $C$ is linear with respect to $C$, and the condition of not being a full rank in the space of \emph{all} matrices is specified 
by a set of polynomial equations (namely, determinants of all maximal square submatrices should be zero). Thus, we have shown that set $C_{singular} = \lbrace C \in \mathbb{R}^{l \times k \times s \times s} \ | \ \widehat{C} \text{ is not of full rank} \rbrace$ is algebraic; and by the well-known property of algebraic sets there are two options: either $\mu(C_{singular})=0$ or $C_{singular}=\mathbb{R}^{l \times k \times s \times s} $ (with $\mu$ being the standard Lebesgue measure).	To show that the latter does not hold, we provide a concrete example of a weight $C$ not in $C_{singular}$. Namely, consider the following $C$.
\begin{equation}
C[i, j, p, q] = \begin{cases}
\delta_{ij}, \quad p=q=1, \\
0, \quad \text{otherwise}.
\end{cases}
\end{equation}
Here $\delta_{ij}$ denotes the Kronecker delta symbol:
$$
\delta_{ij} = \begin{cases}
1, i=j, \\
0, i \neq j.
\end{cases}
$$
We observe that the corresponding matrix $\widehat{C}$ is of particularly simple structure:
$$\widehat{C}[i, j] = \delta_{ij},$$ which trivially is of full rank. 

\paragraph{Arbitrary stride}
The same argument as before applies. Notice that selection of a bigger stride corresponds to selecting specific rows from the matrix $\widehat{C}$ obtained for stride one. By using the same weight tensor $C$ as in the case of stride one, we find that the obtained matrix $\widehat{C}$ contains $\min(m^2k, n^2l)$ distinct rows of the identity matrix, followed by possible zero rows and thus also has full rank.

\end{proof}

After these preliminary results, we are ready to extend them to the case of arbitrary latent space. Namely, suppose that $z$ is sampled from an arbitrary manifold $\mathcal{M}_z \subset \mathbb{R}^d$. We use the following simple lemma.
\begin{lemma}\label{lemma:restriction}
Let $f:\mathcal{M} \to \mathcal{N}$ be an arbitrary smooth embedding. Let $\mathcal{S} \subset \mathcal{M}$ be a smooth embedded submanifold. Then $f|_{\mathcal{S}}$ is also a smooth embedding.
\end{lemma}
\begin{proof}
The proof follows from the definition. Indeed, for every point $x \in \mathcal{S} \subset \mathcal{M}$ we have $T_x\mathcal{S} \subset T_x \mathcal{M}$ and restriction of the derivative of $f$ onto this subspace is also injective. Note that $f|_{\mathcal{S}}$ is also injective and open map.
\end{proof}

By combining \cref{lemma:linear,lemma:conv,lemma:restriction,prop:embedding} we obtain the following result.
\begin{theorem}\label{thm:top_inv}
Let $f_{\theta}(z)$ be an arbitrary neural network consisting of expanding fully connected layers and expanding convolutions, and $z \in \mathcal{M}_z \subset \mathbb{R}^d$. Denote \mbox{$\mathcal{M}_{\theta} = f_{\theta}(\mathcal{M}_z)$}. Then for all parameters $\theta$ but measure zero the following properties hold:
\begin{itemize}
	\item  $\mathcal{M}_{\theta}$ is a smooth embedded manifold;
	\item $\mathcal{M}_{\theta} \simeq \mathcal{M}_z$.
\end{itemize}
\end{theorem}
\begin{proof}
Theorem follows from \cref{lemma:linear,lemma:conv,lemma:restriction,prop:embedding} and the fact that a composition of embeddings is also an embedding.
\end{proof}
For many datasets used in practice, it seems very unlikely that the data comes from manifolds with very simple topological properties, as even basic visual patterns may possess quite non-trivial topological structure \citep{ghrist2008barcodes}. Thus on the first sight, it seems that \cref{thm:top_inv} suggests that using only expanding architectures, it is impossible to approximate an arbitrary data manifold with latent space being $\mathbb{R}^d$ (or an open unit cube). Such models are, however, extremely successful in practice. While we do not provide a precise theorem for this case, based on the discussion in \cref{sec:cyclic}, we hypothesize that it may possible to approximate an arbitrary compact data manifold using expanding networks up to a subset of \emph{arbitrary small measure}, and thus limitations imposed by \cref{thm:top_inv} are negligible in practice.



\section{Cycle generative models}\label{sec:cyclic}

Another popular class of models used for instance for the unsupervised image to image translation \citep{zhu2017unpaired,isola2017image} learn a mapping along with its inverse from one data manifold to another. We specify this task as follows. Given two data manifolds $\mathcal{M}$ and $\mathcal{N}$, the goal is two train two neural networks $f_{\theta}(x)$ and $g_{\phi}(y)$ such that $f_{\theta}(x)$ is a diffeomorphism of $\mathcal{M}$ and $\mathcal{N}$ with $g$ being inverse of $f$.

First of all, let us notice that we do not expect for such $f$ and $g$ to exist for two general manifolds since two manifolds of different topological properties cannot be diffeomorphic. However, based on \cref{thm:top_inv} we expect that the desired properties may hold \emph{approximately}. Let us start with lemmas ensuring existence of functions $f$ and $g$ which map $\mathcal{M}$ \emph{approximately} to $\mathcal{N}$ and $\mathcal{N}$ \emph{approximately} to $\mathcal{M}$ correspondingly. In this section, we again consider only the case of smooth data manifolds.


First of all, we recall the following result \citep{sakai1996riemannian}, proved in a very similar manner to \cref{lemma:map}.
\begin{lemma}\label{lemma:open-dense}
	Every compact connected $d$-dimensional manifold $\mathcal{M}$ contains an open dense set diffeomorphic to $\mathbb{R}^d$. Moreover, complement of this set has measure zero in $\mathcal{M}$.
\end{lemma}
We use this result to obtain the following lemma.
\begin{lemma}\label{lemma:subsets}
For every $\delta > 0$ there exist compact subsets $\mathcal{M}_{\delta} \subset \mathcal{M}$ and $\mathcal{N}_{\delta} \subset \mathcal{N}$ such that $\mu(\mathcal{M}\setminus \mathcal{M}_{\delta}) < \delta $ and $\mu(\mathcal{N}\setminus \mathcal{N}_{\delta}) < \delta $ and $\mathcal{M}_{\delta}$ is diffeomorphic to $\mathcal{N}_{\delta}$.
\end{lemma}
\begin{proof}
	For each of the manifolds $\mathcal{M}$ and $\mathcal{N}$ select the open dense set of full measure as in \mbox{\cref{lemma:open-dense}}. Each of these subsets is diffeomorphic to an open unit ball in $\mathbb{R}^d$ via maps $h_\mathcal{M}$ and $h_\mathcal{N}$. In order to construct $\mathcal{M}_{\delta}$ and $\mathcal{N}_{\delta}$ it sufficient to take preimages under $h_\mathcal{M}$ and $h_\mathcal{N}$ correspondingly of a sufficiently large closed ball $B_r$ (as with $r\to1$ we have $\mu(h^{-1}_\mathcal{M}(B_r))\to \mu(\mathcal{M})$ and \mbox{$\mu(h^{-1}_\mathcal{N}(B_r))\to \mu(\mathcal{N})$}). 
\end{proof}
We are now ready to provide our main result on cycle generative models.
\begin{theorem}[Geometric Universality for Cycle Models]
	Fix any two compact connected manifolds  $\mathcal{M}$ and $\mathcal{N}$ of the same dimension and a nonconstant, bounded, continuous nonlinearity $\sigma(x)$. Then for every $\delta > 0$ and  $\varepsilon > 0$ there exist compact subsets $\mathcal{M}_{\delta} \subset \mathcal{M}$ and $\mathcal{N}_{\delta} \subset \mathcal{N}$ and a pair of feedforward neural networks $f_{\theta}(x)$, $g_{\phi}(y)$ with the activation function $\sigma(x)$ satisfying the following conditions:
	\begin{itemize}
		\item $\mu(\mathcal{M}\setminus \mathcal{M}_{\delta}) < \delta $ and $\mu(\mathcal{N}\setminus \mathcal{N}_{\delta}) < \delta $;
		\item $d_H(f_{\theta}(\mathcal{M}_{\delta}), \mathcal{N}_{\delta})  < \varepsilon$ and $d_H(g_{\phi}(\mathcal{N}_{\delta}), \mathcal{M}_{\delta})  < \varepsilon$;
		\item $\|g_{\phi} \circ f_{\theta} - id\|_{\mathcal{M}_{\delta}} < C\varepsilon$ and $\|f_{\theta} \circ g_{\phi} - id\|_{\mathcal{N}_{\delta}} < C\varepsilon$ with constant $C$ depending only on manifolds $\mathcal{M}$ and $\mathcal{N}$.
	\end{itemize}
\end{theorem}
\begin{proof}
Let us start by selecting subsets $\mathcal{M}_{\delta}$ and $\mathcal{N}_{\delta}$ and a diffeomorphism $f:\mathcal{M}_{\delta} \to \mathcal{N}_{\delta}$ along with its inverse $g$ as specified by \cref{lemma:subsets}. For simplicity let us also assume that $\mathcal{M} \subset I_n$ and $\mathcal{N} \subset I_n$. By means of the Whitney extension theorem \citep{whitney1934analytic} we can smoothly extend $f$ and $g$ to the entire cube $I_n$, and apply the universal approximation theorem  \citep{hornik1991approximation}, thus obtaining two feedforward neural networks $f_{\theta}(x)$ and $g_{\phi}(y)$ such that
\begin{equation}\label{eq:f}
	\|f_{\theta}- f\|_{I_n} < \varepsilon,
\end{equation}
and
\begin{equation}\label{eq:g}
\|g_{\phi}- g\|_{I_n} < \varepsilon,
\end{equation}
with all the functions defines on the unit cube $I_n$.
This proves first two points in the theorem. To show the last property we find that $\forall x \in \mathcal{M}_{\delta}$ the following estimate holds.
\begin{equation}
\begin{split}
\|g_{\phi} \circ f_{\theta}(x) - x\| & = \|g_{\phi} \circ f_{\theta}(x) - g\circ f_{\theta}(x) + g\circ f_{\theta}(x) - x \| \\
& \leq \|g_{\phi} \circ f_{\theta}(x) -  g\circ f_{\theta}(x) \| +  \|g \circ f_{\theta}(x) - x\| \\
& \leq \varepsilon + \|g \circ f_{\theta}(x) - g \circ f(x)\| \\
& \leq \varepsilon + \max_{\mathcal{M}_{\delta}}\|Dg\| \|f_{\theta}(x) -  f(x)\| \\
& \leq (1 + \max_{\mathcal{M}_{\delta}} \|Dg\|)\varepsilon,
\end{split}
\end{equation}
where he have used the fact that $g \circ f (x) = x$ for $x\in\mathcal{M}_{\delta}$ and property \eqref{eq:g}. The second part of the claim is proved similarly.
\end{proof}
Neural networks $f_{\theta}$ and $g_{\phi}$ constructed in the proof perform \emph{translation} from data sampled from $\mathcal{M}_{\delta}$ to data coming from approximately $\mathcal{N}_{\delta}$, and existence of such networks for arbitrary manifolds may partially explain huge empirical success of cyclic models. Even though the theorem is valid for an arbitrary pair of manifolds, we hypothesize that for datasets containing visually similar images such a map may be much easier to model, than for two arbitrary manifolds without such a connection.
\section{Conclusion and future work}
In this work we have attempted to partially explain huge empirical success of generative models. Our results show only existence of neural networks approximating arbitrary manifolds, and do not specify how one can estimate the size of a network required for any given manifold. We hypothesize, however, that there might exist a connection between certain geometrical properties of a manifold (curvature, various topological properties), and the width/depth of a neural network required. One interesting direction of research left for a future work is analyzing this relation for datasets popular in computer vision, such as MNIST or CelebA, or toy datasets sampled from simple small dimensional manifolds (tori, circles), where one can easily vary the topological properties.



\bibliography{main}
\small {\bibliographystyle{plainnat}}

\end{document}